\documentclass[runningheads,envcountsame]{llncs}
\usepackage{graphicx}
\usepackage{hyperref}

\usepackage{amsxtra, amsfonts, amssymb, amstext}
\usepackage{mathtools}
\usepackage{nicefrac}
\usepackage{xspace}
\usepackage{cite}

\usepackage{enumerate}
\usepackage{bm}
\usepackage{bbm}
\usepackage{xcolor}

\usepackage[linesnumbered,ruled,vlined]{algorithm2e}

\renewcommand{\phi}{\ensuremath{\varphi}}

\newcommand{\whp}{w.h.p.\xspace}
\newcommand{\wrt}{w.r.t.\xspace}
\newcommand{\ie}{i.e.\xspace}

\newcommand{\pimp}{p_{\text{imp}}}

\newcommand{\onemax}{\textsc{OneMax}\xspace}
\newcommand{\zeromax}{\textsc{ZeroMax}\xspace}
\newcommand{\OM}{\textsc{OM}\xspace}
\newcommand{\ZM}{\textsc{ZM}\xspace}
\newcommand{\disOM}{\textsc{disOM}\xspace}

\newcommand{\cliff}{\textsc{Cliff}\xspace}

\newcommand{\hurdle}{\textsc{Hurdle}\xspace}
\newcommand{\BBFunnel}{\textsc{BBFunnel}\xspace}

\newcommand{\oea}{${(1+1)}$-EA\xspace}
\newcommand{\sea}{SA-${(1, \lambda)}$-EA\xspace}
\newcommand{\sear}{SA-${(1, \lambda)}$-EA with resets\xspace}

\newcommand{\olea}{$(1 + \lambda)$-EA\xspace}

\newcommand{\oclea}{$(1 , \lambda)$-EA\xspace}

\newcommand{\N}{{\mathbb N}}

\newcommand\eps{\varepsilon}

\newcommand{\CC}{{\mathcal C}}
\newcommand{\CD}{{\mathcal D}}

\begin{document}
\title{Self-Adjusting Evolutionary Algorithms Are Slow on Multimodal Landscapes}
\titlerunning{Self-Adjusting Offspring Populations Size}
%
\author{Johannes Lengler \and Konstantin Sturm}
\authorrunning{J. Lengler, K. Sturm}
%
\institute{Department of Computer Science, ETH Zürich, Zürich, Switzerland\\
\email{\{johannes.lengler,konstantin.sturm\}@inf.ethz.ch}}
\maketitle              
\begin{abstract}
The one-fifth rule and its generalizations are a classical parameter control mechanism in discrete domains. They have also been transferred to control the offspring population size of the \oclea. This has been shown to work very well for hill-climbing, and combined with a restart mechanism it was recently shown by Hevia Fajardo and Sudholt to improve performance on the multi-modal problem \cliff drastically. 

In this work we show that the positive results do not extend to other types of local optima. On the distorted OneMax benchmark, the self-adjusting \oclea is slowed down just as elitist algorithms because self-adaptation prevents the algorithm from escaping from local optima. This makes the self-adaptive algorithm considerably worse than good static parameter choices, which do allow to escape from local optima efficiently. We show this theoretically and complement the result with empirical runtime results. 

\end{abstract}
\section{Introduction}

Evolutionary algorithms (EAs) are a class of randomized optimization heuristics that are popular because they are flexible and can be widely applied. It is desirable for such general-purpose optimization algorithms to be as easy to use as possible. Thus, an important goal in designing EAs is to reduce the number of hyper-parameters that need to be set by the user. A convenient way is to make the algorithms \emph{self-adjusting}, i.e., to add mechanisms that dynamically adapt the hyper-parameters in an automatic way. This approach has some other advantages. Sometimes there is no static parameter setting which is optimal throughout the whole optimization process, in which case self-adjusting mechanisms can be superior~\cite{doerr2021runtime,doerr2021self,hevia2024cliff}. 

A self-adaptation mechanism that has received increasing attention in recent years is the \emph{one-fifth rule} and its generalization, the $(1:s+1)$-rule. This is a classical rule in the domain of continuous optimization~\cite{rechenberg1978evolutionsstrategien}, but in the last years it has also been successfully transferred to discrete domains~\cite{doerr2021self,hevia2024cliff,hevia2023success,kaufmann2023hardest,kaufmann2023self,kaufmann2024onemax}, see also the reviews in~\cite{doerr2020theory} and~\cite{hevia2023success}. The $(1:s+1)$ rule may be used to control a hyper-parameter that regulates the trade-off between efficiency and the \emph{success rate}, which is the probability of making an improvement in one generation. It defines a \emph{target success rate}, which is $1/s$ in the case of the $(1:s+1)$ rule. Then, whenever a generation is successful it adapts the hyper-parameter to improve efficiency at the cost of a smaller success rate. For unsuccessful generations, it adapts the hyper-parameter in the other direction. Both adjustments are balanced in such a way that the success rate is pushed toward the target success rate. Some hyper-parameters for which this rule has been shown to work particularly well are the \emph{step size} in continuous optimization~\cite{kern2004learning}, the \emph{mutation rate} in discrete domains~\cite{doerr2021runtime}, and the \emph{offspring population size} for hill-climbing tasks~\cite{hevia2023success,kaufmann2023self}.

This work will focus on the offspring population size, specifically the \sea. The algorithm generates $\lambda$ offspring from the same parent in each generation, and proceeds the best offspring as parent for the next generation. It adapts the offspring population size $\lambda$ with the $(1:s+1)$ rule, see Section~\ref{sec:sear} for details. Recently, some very positive results could be shown for this algorithm. Hevia Fajardo and Sudholt studied the \sea on \onemax\footnote{\onemax is defined on the hypercube $\{0,1\}^n$ and assigns to each bit string $x$ the number of one-bits in $x$.}, a benchmark in which progress gets harder during the optimization process. They could show that for $s\le 1$, the $(1:s+1)$ rule automatically chooses and maintains the optimal $\lambda$ throughout optimization, ranging from constant $\lambda$ at the beginning to almost linear $\lambda$ as the algorithm approaches the optimum. Kaufmann, Larcher, Lengler, and Zou extended this result (for smaller $s$) to all monotonic functions, showing that the \sea shows optimal parameter control on every monotonic function. These results show that the \sea can be very successful on hill-climbing tasks without local optima.

In principle, the \oclea is also well-suited to deal with local optima. In fact, the comma strategy allows the \oclea to escape local optima by ``forgetting'' the parent, other than its elitist counterpart \olea, in which the parent always competes for entering the next generation. Indeed, this makes the \oclea more efficient than the \olea in landscapes with planted local optima~\cite{jorritsma2023comma}. However, a priori the $(1:s+1)$ rule is misaligned with this escaping option. When the algorithm is stuck in a local optimum, the $(1:s+1)$ rule starts increasing the offspring population size, which is the correct behavior for hill-climbing. However, this also increases the probability of producing a clone of the parent among the offspring, in which case the algorithm mimics the behavior of the plus strategy and loses its ability to escape local optima. When $\lambda$ is of logarithmic size or larger, the \oclea degenerates into the \olea. Thus, we can not hope that the \sea may be suited for local optima in its standard form.

To avoid this problem, Hevia Fajardo and Sudholt proposed as a simple fix to \emph{restart} the offspring population size at $\lambda=1$ whenever it exceeds some threshold $\lambda_{\max}$. In a spectacular result for the notoriously hard benchmark \cliff, which features a large plateau of local optima, they could show that the \sear optimizes cliff with $O(n\ln n)$ function evaluations, not substantially slower than \onemax~\cite{hevia2024cliff}. This is not only much better than any known performance of elitist algorithms on \cliff, but it is also drastically faster than the \oclea with any \emph{static} parameter $\lambda$, which needs time $\Omega(n^{3.98})$ even for optimally chosen static $\lambda$~\cite{hevia2024cliff}. These results gave hope that the \sear may be able to provide optimal strategies for a wide range of fitness landscapes. Unfortunately, in this paper we show that this algorithm has some severe limitations when the local optima are not clustered in form of a large cliff, but rather scattered throughout the fitness landscape. While we do believe that the \sear deserves its place in modern optimization portfolios, our result shows that it is no panacea.

\subsection{Our Result}\label{sec:our_results}
We study the \sear on the \emph{distorted OneMax} benchmark \disOM in a fixed-target setting. The function \disOM is obtained from the \onemax benchmark by increasing the fitness of each search point with some probability $p$ by some value $d>1$, thus planting local optima at random places of the landscape. For the formal definition, see Section~\ref{sec:disOM}. We mostly take the parameters of \disOM from~\cite{jorritsma2023comma}; in fact, we even allow slightly more general parameters. In particular, we choose $p= \omega(1/n\ln n)$ to make sure that the algorithms encounter distorted points during optimization, and we choose the fixed target in such a way that the target can be reached efficiently with some static values of $\lambda$, see~\cite{jorritsma2023comma} for a more thorough discussion.

In~\cite{jorritsma2023comma} it was shown that the \olea is slowed down by a factor of $1/p$, yielding runtime $\Omega(n\ln n/p)$. This can be substantial since $1/p$ may be an almost linear factor. The algorithm is slowed down because the plus strategy is not able to escape local optima and thus needs to hop from one local optimum to the next. This makes it by a factor $p$ harder to find an improvement since the algorithm does not only need to create an offspring of larger \onemax value, but in addition this offspring must be distorted. On the other hand, the \oclea with static $\lambda$ is unaffected because it can efficiently escape from local optima and has the same runtime $O(n\ln n)$ as for the corresponding \onemax problem. 

We show that the \sear suffers the same performance loss as the \olea on \disOM: it needs time $\Omega(n\ln n/p)$ to reach the fitness target. Thus, the self-adjusting mechanism costs performance and slows down the algorithm by a factor of $1/p$ compared to the known runtime $O(n\log n)$ of the \oclea with \emph{static} $\lambda$~\cite{jorritsma2023comma}.
    \begin{theorem}
        \label{thm:main}
        Consider the \sea with a resetting mechanism for the offspring population size on \disOM with $p = \omega(1/n\ln{n})$, $d=\Omega(\ln{n})$, and $\lambda_{max} \ge n^{\Omega(1)}/p$. With high probability the algorithm takes $\Omega(n \ln{n}/p)$ function evaluations to reach a target fitness of $n-k^\ast$ for $k^\ast = n^{1-\Omega(1)}$.
    \end{theorem}
We put this into context with the known results on the \olea and the \oclea. The \oclea with an optimal \emph{static} $\lambda$ is faster than both the self-adjusting one and the \olea in the specific setting presented in \cite{jorritsma2023comma}. The following Corollary summarizes these findings.
\begin{corollary}\label{cor:results}
    Let $k^\ast = n^{\Omega(1)}\cap n^{1-\Omega(1)}$, $p =\omega(1/(n\ln n))$, and assume that there is a constant $\eps >0$ such that
    \begin{align}\label{eq:range-of-q}
        p \le (k^\ast/n)^{1+\eps}.
    \end{align}
    Finally, assume that $d = \Omega(\ln n)$ with $d\le k^{\ast}$ and $\lambda_{\max}= n^{\Omega(1)}$. Then \whp on \disOM the number of evaluations to reach fitness at least $n-k^\ast$ is
    \begin{enumerate}
        \item \label{cor:results_item_1}$\Omega(n\ln n/p)$ for the \olea with any static $\lambda \ge 1$,
        \item \label{cor:results_item_2}$\Omega(n\ln n/p)$ for the \sea,
        \item \label{cor:results_item_3}$O(n\ln n)$ for the \oclea with a suitable static $\lambda = \Theta(\ln n)$.
    \end{enumerate}
\end{corollary}
The difference is a factor of $1/p$, which can be substantial, for example settings with $1/p=\Omega(n)$ are included. The conditions come directly from~\cite[Assumption~1.4]{jorritsma2023comma}, and the third statement comes from~\cite[Theorem~1.1]{jorritsma2023comma}. We note that a sufficient condition for $\lambda$ to make the \oclea efficient 
was given in~\cite{jorritsma2023comma} as $\lambda$ which satisfy $(1+\delta)\log_{e/(e-1)}(1/p) \le \lambda \le (1-\delta)\log_{e/(e-1)}(n/k^\ast)$ for an arbitrary constant $\delta >0$.\footnote{The authors of~\cite{jorritsma2023comma} comment that the second condition may not be necessary. For more details on the parameters we refer to the discussion in~\cite{jorritsma2023comma}.}

We further corroborate the theoretical findings with some empirical results, which are presented in Section~\ref{sec:experiments}. Those show quite clearly an asymptotic of $\Theta(n\ln n/p)$ for the \sea, so our lower bound in Theorem~\ref{thm:main} is apparently tight. Moreover, they confirm the asymptotic statement from Corollary~\ref{cor:results} that the \oclea with static $\lambda$ is much faster than both the \olea with static $\lambda$ and the \sea. 

We want to emphasize that we do not advertise abolishing self-adaptation and returning to static parameter choices. While we show that there are regimes in which the self-adjusting algorithm is slow, there are also other regimes where static choices have disadvantages. In particular, if the target fitness is large (e.g., $k^\ast=0$) then there is no static $\lambda$ which can reach the target fitness and at the same time avoid being stuck in local optima for long. Thus, future research should aim for alternatives which can avoid the downsides from both approaches. Moreover, for practical matters, general-purpose optimizers like EAs should always be used as part of a portfolio of optimization techniques which does not hinge on a single algorithm. 

\section{Notation and Preliminaries}\label{sec:prelim}
    \paragraph{General Notation.}
    We write $[n] \coloneqq \{1,..., n\}$. Search points are denoted by $x = (x_1, ...,x_n) \in \{0,1\}^n$, and the \onemax value is $\OM(x) \coloneqq \sum_{i \in [n]} x_i$. The \zeromax function is defined as $\ZM(x) \coloneqq n- \OM(x)$. We denote the all-one-string by $\vec{1} = (1,..., 1)$. For $x, y \in \{0,1\}^n$, the Hamming distance $H(x,y)$ of $x$ and $y$ is the number of positions $i \in [n]$ such that $x_i \neq y_i$. We denote the natural logarithm of $n$ by $\ln{n}$. With high probability (\whp) means with probability $1- o(1)$ for $n \rightarrow \infty$. For a real number $a$, we denote by $\lfloor a \rceil := \lfloor a+1/2 \rfloor$ the closest integer to $a$.

    \paragraph{Distorted OneMax.}\label{sec:disOM}
    The function is formally defined as $\disOM : \{0,1\}^n \rightarrow \mathbb{R}_{\ge 0}$. 
    We partition the search space $\{0,1\}^n$ into two sets $\CC$ and $\CD$ of ``clean'' and ``distorted'' points, respectively. For each $x \in \{0,1\}^n$ we have $x\in \CD$ with probability $p$ and $x \in \CC$ otherwise, independently of the other points. We define \disOM as
    \begin{displaymath}
        \disOM(x) \coloneqq \OM(x) + \begin{cases} d & \text{ if $x\in \CD$,} \\ 0 & \text{otherwise.}\end{cases}
    \end{displaymath}
    The function was introduced in~\cite{jorritsma2023comma}, where it was shown that plus strategies are slowed down by a factor of $1/p$, while comma strategies are not affected. Very recently, it was shown that this effect is even more drastic when the height of the distortion is drawn randomly for every distorted point, making the plus strategies super-polynomially slow~\cite{lengler2024plus}. However, we will follow the setup in~\cite{jorritsma2023comma} and use the same offset $d$ for all distorted points.

    \paragraph{Algorithms.} \label{sec:sear}
    We consider the \sear, which is identical to the one presented in \cite{hevia2024cliff}. The algorithm is called self-adjusting as the offspring population size $\lambda$ is adapted in each generation. A fitness increase results in a decrease of $\lambda$ to $\lambda / F$ for some $F>1$. If the fitness stays the same or decreases, the generation is called unsuccessful, and $\lambda$ is increased to $\lambda \cdot F^{1/s}$ for some $s>0$. Throughout the paper we will assume that $F,s$ are constant and $0<s<1$. When a sequence of $s$ successful generations and a single unsuccessful one occur, the final offspring population size is unchanged due to $\lambda \cdot (F^{1/s})^s \cdot (1/F) = \lambda$, which follows the previously mentioned $(1: s+1)$-success rule \cite{kern2004learning}. The ``reset part'' in the algorithm's name refers to the maximum offspring population size $\lambda_{\max}$ we impose. If a generation with $\lambda = \lambda_{\max}$ is unsuccessful, the offspring population size is reset to 1 instead of being increased further.    
    
    \begin{algorithm}[H]
        \DontPrintSemicolon
        \caption{Self-adjusting $(1, \lambda)$ EA resetting $\lambda$ for maximizing $f$ to target $n-k^\ast$.}

        \SetKwInput{KwInitialization}{Initialization}
        \SetKwInput{KwOptimization}{Optimization}
        \SetKwInput{KwMutation}{Mutation}
        \SetKwInput{KwSelection}{Selection}
        \SetKwInput{KwUpdate}{Update}

        \KwInitialization{$t= 0$; choose $x_0 \in \{0,1\}^n$ uniformly at random and $\lambda_0 = 1$;}
        \KwOptimization{
            \While{$f(x_t) < n-k^\ast$}{
                \KwMutation{
                    \For{$i \in \{1,...,\lfloor \lambda_t \rceil\}$}{
                        $y_t^{i} \in \{0,1\}^n \leftarrow \text{standard bit mutation($x_t$)}$; \;
                    }
                }
                \KwSelection{
                    Let $y_t = \arg\max\{f(y_t^{1}),..., f(y_t^{\lfloor \lambda_t \rceil})\}$, breaking ties uniformly at random; \;
                }
                \KwUpdate{
                    $x_{t+1} \leftarrow y_t$; \;
                    \lIf{$f(x_{t+1}) > f(x_t)$}{
                        $\lambda_{t+1} \leftarrow \max\{\lambda_t / F, 1\}$;
                    }
                    \lIf{$f(x_{t+1}) \leq f(x_t) \wedge \lambda_t = \lambda_{\max}$}{
                    $\lambda_{t+1} \leftarrow 1$;
                    }
                    \lIf{$f(x_{t+1}) \leq f(x_t) \wedge \lambda_t \neq \lambda_{\max}$}{
                    $\lambda_{t+1} \leftarrow \min\{\lambda F^{1/s}, \lambda_{\max}\}$;
                    }
                    $t\leftarrow t+1$;
                }
            }
        }
    \end{algorithm}

    New offspring of a search point $x$ are created by applying a \emph{standard bit mutation}: Each bit in $x$ is being flipped independently with probability $1/n$. We consider a fixed target setting of $n-k^\ast$ following \cite{jorritsma2023comma}.
    
\section{Properties of the \sea}\label{sec:estimates}
    In this section, we provide a series of useful probability estimates, most of which are not specific to our benchmark algorithm combination and may prove useful in other settings as well. 

    We call an offspring a \emph{clone} of the parent if it is an exact copy. The first lemma provides bounds on (not) having a clone among the offspring.
    \begin{lemma}
    \label{lemma:prob-no-clone}
        The probability of not having a clone of the parent among $\lambda_t$ offspring is at least $\left(1- 1/e \right)^{\lambda_t}$. The probability of having at least one clone among the offspring is at least $\exp{(-en/(\lambda_t (n-1)))}$.
    \end{lemma}
    \begin{proof}
        The probability that a single offspring is a clone is $(1- 1/n)^n$. The probability that none of the offspring is a clone is therefore at least 
        \begin{displaymath}
            \left(1- \left(1-\frac{1}{n}\right)^{n}\right)^{\lambda_t} 
            \ge \left(1-\frac{1}{e}\right)^{\lambda_t}.
        \end{displaymath}
        The probability of at least one clone is at least 
        \begin{align*}
            &1- \left(1- \left(1-\frac{1}{n}\right)^{n}\right)^{\lambda_t} 
            \ge 1- \left(1-\frac{n-1}{en}\right)^{\lambda_t}
            \ge 1- e^{-\lambda_t (n-1)/en} \\
            &\ge 1 - \frac{1}{1+\lambda_t (n-1)/en}
            = \frac{1}{1+ en/(\lambda_t (n-1))}
            \ge \exp{\left(-\frac{en}{\lambda_t (n-1)}\right)}. \tag*{\qed} 
        \end{align*}
    \end{proof}
    The next result is more specific to \disOM. It will allow us to argue that if each generation has a distorted point among its offspring, the algorithm will not leave the set of distorted points $\CD$.
    \begin{lemma}
        \label{lemma:accept_dist_offspring}
        Let the size of the distortion be $d = \Omega(\ln{n})$. For any constant $c>0$, if any of the $\lambda = n^c$ offspring is distorted, with probability $1-n^{-\omega(1)}$ the accepted offspring will also be distorted.
    \end{lemma}
    \begin{proof}
        For a clean offspring to be accepted, its fitness must be larger than that of the distorted point. Since $d \in \Omega(\ln{n})$, either the distorted or the clean point must have Hamming distance $\Omega(\ln{n})$ to the parent. By \cite[Lemma~3.3]{jorritsma2023comma} the probability of a single offspring satisfying this is $n^{- \Omega(\ln{\ln{n}})}$. With a union bound over the $n^{c}$ offspring, the lemma follows.
    \qed    
    \end{proof}
    A major complication in runtime analyses of algorithms on \disOM is the fact that the noise of the benchmark is frozen \cite{lengler2024plus, jorritsma2023comma}. This means that when we sample an offspring, we can not simply assume that it is distorted with probability $p$ since it may have been sampled earlier. The probability of having been sampled may depend on the history of the run. This may seem like a technicality, but in fact it may have a major impact on the resulting runtimes, see \cite{jorritsma2023comma} for a discussion. Following a technique invented in \cite{lengler2024plus}, we prove the following lemma to show that enough neighbours of the current search point have not been queried yet and thus provide ``fresh randomness''.
    \begin{lemma}
        \label{lemma:prob-p}
        Consider any algorithm that creates offspring from previously visited search points with standard bit mutation of mutation rate $1/n$ on any fitness function. Assume the algorithm has created $o(n^3)$ offspring so far. Let $x\in\{0,1\}^n$ be any search point. Then the probability that a random offspring of $x$ has not yet been queried is $\Omega(1)$. In particular, for \disOM the probability that this offspring is distorted is $\Omega(p)$.  
    \end{lemma}
    \begin{proof}
        The probability of an offspring having a Hamming distance of exactly 3 to the parent is
        \begin{displaymath}
            {\binom{n}{3}} \left(\frac{1}{n}\right)^3 \left(1- \frac{1}{n}\right)^{n-3} 
            \ge \left(\frac{n}{3}\right)^3 \left(\frac{1}{n}\right)^3 \left(1- \frac{1}{n}\right)^{n-1}
            \ge \frac{1}{27e}
            = \Omega(1).
        \end{displaymath}   
        For a parent, there are ${\binom{n}{3}}$ points in the three-neighborhood. Since we assumed $o(n^3)$ points have been queried so far, each new offspring has not been sampled before with probability $\Omega(1)$, conditional on having Hamming distance three to the parent. Together, this implies that with probability $\Omega(1)$ each new offspring has not been sampled before. The second claim follows immediately because a search point that has not been sampled before is distorted with probability $p$.
        \qed
    \end{proof}
    The final lemma enables us to contend that a significant drop in $\lambda$ is unlikely in certain settings. It is closely related to the Gambler's Ruin Problem \cite{feller1991introduction}.
    \begin{lemma}
        \label{lemma:interval-crossing}
        Consider the \sea on \disOM in a state $(x_0, \lambda_0)$ with $\lambda_0 \in [\alpha F^{\beta-1}, \alpha F^{\beta})$, for some $\alpha \ge 1$ and $\beta \in \N$. Suppose there is $q>0$ and a set $X\subseteq \{0,1\}^n \times [1,\lambda_{\max}]$ of states such that for any state $(x_t, \lambda_t) \in X$, the probability that the next generation is successful is at most $q$. 
        Then the probability that $\lambda$ falls below $\alpha$ before either the algorithm leaves $X$ or it increases $\lambda$ to at least $\alpha F^{\beta}$ is at most $(1/q-2)/(\left(1/q-1\right)^{\beta+1}-1)$.
    \end{lemma}
    \begin{proof}
        We approach the problem from a random walk perspective and aim to build a connection to the Gambler's Ruin problem \cite{feller1991introduction}. Let us define the states $S_0, S_1, ..., S_{\beta+1}$. The algorithm is in state $S_i$ if the current offspring population size $\lambda_t$ is in the interval $[\alpha F^{i-1}, \alpha F^i)$, and $x_t$ is arbitrary in $X$. Notice that when in $S_i$, $i$ consecutive successful generations reduce the offspring population size from $\lambda_t$ to $\lambda_t F^{-i} < \alpha F^{i} F^{-i} = \alpha$. Let $P_i$ denote the probability of the algorithm reaching the state $S_{0}$ before leaving $X$ or reaching $S_{\beta+1}$. Clearly $P_0 = 1$ and $P_{\beta+1} = 0$. We proceed to compute $P_i$ for $1 \leq i \leq \beta$. We will pessimistically assume that the algorithm does not leave the set $X$. 

        If the current state is $S_i$ and we encounter a successful generation, the algorithm moves to $S_{i-1}$. Conversely, if the generation is not successful $\lambda$ is multiplied by $F^{1/s}$ resulting in a move to a state $S_{k}$ with $k > i$. By the bound $q$ on the probability of successful generation, we can bound $P_i$ recursively\footnote{It could happen that the $P_i$ are not increasing due to the set $X$, in which case the bound may not hold. However, the recursion is correct if we replace $P_i$ with the minimal probability over all possible search spaces that satisfy the condition of the lemma because those probabilities are increasing. We suppress this complication from the proof.} as
        \begin{displaymath}
            P_i \leq q P_{i-1} + (1-q) P_{k}.
        \end{displaymath}
        Pessimistically assuming  $k = i+1$ and equality, we obtain the classical recursion for some upper bound $\tilde P_i \ge P_i$,
        \begin{displaymath}
            \tilde P_i = q \tilde P_{i-1} + (1-q) \tilde P_{i+1}.
        \end{displaymath}
        The above equation is the recursion for the Gambler's Ruin Problem, which has the following solution~\cite{feller1991introduction} for $0<i\leq \beta$: 
        \begin{displaymath}
            \tilde P_{i} = \frac{1 - \left((1-q)/q\right)^{\beta + 1-i}}{1 - \left((1-q)/q\right)^{\beta+1}}.
        \end{displaymath}
        The lemma assumes the offspring population size is in the interval $[\alpha F^{\beta-1}, \alpha F^\beta)$. Thus, recalling $P_\beta$ represents the measure we are looking for:
        \begin{displaymath}
            P_{\beta} \le \tilde P_{\beta} = \frac{1 - \left((1-q)/q\right)}{1 - \left((1-q)/q\right)^{\beta+1}}
            = \frac{1/q-2}{\left(1/q-1\right)^{\beta+1}-1} \tag*{\qed} 
        \end{displaymath}
    \end{proof}
    
\section{Lower Runtime Bounds}\label{sec:analysis}
    In this section, we prove Theorem~\ref{thm:main}. The core idea is to show that the algorithm must traverse a \onemax-interval $I$ of size $n^{\eps}$ and that it requires $\Omega(n \ln{n}/p)$ evaluations to do so. The $\Omega(1/p)$ factor stems from the observation that the algorithm will stay among the distorted points while crossing the interval.

    To ensure that the algorithm stays among search points in $\CD$ throughout the interval $I$, we need to show that the algorithm enters it in a distorted point. In light of this, we look at a preceding interval $I'$ of Hamming distances to $\Vec{1}$ and show that the algorithm reaches a state in $I'$ such that the search point is distorted and $\lambda$ is at least logarithmic in size.
    \begin{lemma}
        \label{lemma:big-lambda-distorted}
        Consider the \sear on \disOM as in Theorem~\ref{thm:main}. Let $\eps > 0$ be a small constant such that $\lambda_{\max} \ge n^{\eps}$ and $k^\ast + d \leq n^{1- \eps}$, and let $I' \coloneqq [n^{1- \eps/2}, n^{1- \eps/4}]$ be an interval of Hamming distances to $\Vec{1}$. 
        Then with high probability the algorithm will either make $\Omega(n \ln{n}/p)$ evaluations before reaching a search point of distance smaller than $n^{1-\eps/2}$ from $\Vec 1$, or it reaches a state $(x_t, \lambda_t)$ where $x_t$ is in $I'$ and distorted, and $\lambda_t \ge 6e F^{16/\eps}\ln{n}$.
    \end{lemma}
    \begin{proof}
        The lemma holds trivially if the algorithm takes $\Omega(n \ln{n}/p)$ evaluations to cross $I'$. Going forward, we thus assume the algorithm makes $o(n \ln{n}/p)$ evaluations. We show that \whp,
        \begin{enumerate}[(i)]
            \item $\lambda$ is increased to at least $\gamma \coloneqq 6e F^{32/\eps} \ln{n}$ in $O(\ln^2{n})$ evaluations if the algorithm stays in $I'$,
            \item if $\lambda \ge \gamma$ it will not drop below $6eF^{16/\eps}\ln{n}$ before entering a distorted point, and
            \item the algorithm makes $\Omega(n \ln{n})$ evaluations in consecutive generations in $I'$.
        \end{enumerate}
        From these three items, the lemma follows. By (i), $\lambda$ is increased to $\gamma$ in $O(\ln^2{n})$ evaluations. Once the algorithm has reached this offspring population size, \whp it will not drop below $6eF^{16/\eps}\ln{n}$ according to (ii). Together with (iii), this implies $\Omega(n \ln{n}) - O(\ln^2{n}) = \Omega(n \ln{n})$ evaluations are from states in which the offspring population size is at least $6eF^{16/\eps}\ln{n}$. By Lemma~\ref{lemma:prob-p}, each of these offspring have probability $\Omega(p) = \omega(1/(n \ln{n}))$ to be distorted. Hence, the expected number of distorted points among these offspring is $\Omega(p \cdot n\ln{n}) = \omega(1)$. By Lemma~\ref{lemma:accept_dist_offspring}, such an offspring is accepted \whp.
        
        It remains to prove the three items, starting with (i). Let $\lambda_t$ be the current offspring population size. Using a similar argumentation to the proof of \cite[Lemma~4.6]{hevia2024cliff}, $\lceil s\log_{F}{(\gamma/ \lambda_t)} \rceil$ consecutive unsuccessful generations are sufficient to ensure an increase in the offspring populations size to at least $\gamma$. Using $\lfloor x \rceil \leq 2x$ for $x \ge 1$, the number of evaluations this requires is at most 
        \begin{align*}
            \sum_{i=0}^{\left\lceil s\log_{F}{(\gamma/ \lambda_t)} \right \rceil}{\left\lfloor \lambda_t F^{i/s} \right \rceil}
            &\leq 2 \lambda_t \sum_{i=0}^{s\log_{F}{(\gamma/ \lambda_t)} +1}{F^{i/s}}\\
            &\leq 2 \lambda_t \frac{\left(F^{1/s}\right)^{s\log_{F}{(\gamma/ \lambda_t)}+2}-1}{F^{1/s}-1}
            \leq \frac{2F^{2/s}}{F^{1/s}-1}\gamma 
            = O(\ln{n}).
        \end{align*}
        
        To show that we only encounter unsuccessful generations until $\lambda$ is increased to the desired value, we show that \whp none of the $O(\ln{n})$ offspring increase the fitness.
        
        We first consider the case that no distorted offspring is created from a clean parent. Let $\pimp$ be the probability of a single offspring increasing its fitness. $\pimp \leq n^{-\eps/4}$, since at least one of the at most $n^{1-\eps/4}$ zero-bits must be flipped. With a union bound over the $O(\ln{n})$ offspring, the probability of this is $O(\ln{n}) \cdot \pimp = o(1)$.
        
        We turn to the case in which the algorithm creates a distorted offspring from a clean parent before $\lambda$ is increased to at least $\gamma$. Assume the algorithm has just jumped from a clean to a distorted point. Now, it either has to make another such jump within the next $O(\ln{n})$ evaluations, or \whp the algorithm increases $\lambda$ to at least $\gamma$, using the same argumentation as before. For another jump, the algorithm must first leave the set of distorted points $\CD$ again. By showing that with probability $\Omega(1)$ the algorithm will not leave $\CD$ once entered, \whp after $O(\ln{n})$ jumps from clean to distorted points, the algorithm increases the offspring population size to at least $\gamma$. This process takes $O(\ln^2{n})$ evaluations.
        
        It remains to show that the algorithm will stay in the distorted points with probability $\Omega(1)$. A sufficient condition is that each generation has a clone among its offspring. With Lemma~\ref{lemma:prob-no-clone}, the probability of this is at least 
        \begin{align*}
            \prod_{i=0}^{\lceil s\log_{F}{(\gamma/ \lambda_t)} \rceil} \exp{\left(-\frac{en}{\lambda_t F^{i/s} (n-1)}\right)}
            &\ge \exp{\left(-\frac{en}{(n-1)} \sum_{i=0}^{\infty} F^{-i/s}\right)}\\
            &\ge \exp{\left(-\frac{en}{(n-1)} \frac{1}{1-F^{-1/s}}\right)} = \Omega(1).
        \end{align*}
        
        In order to show (ii), note that the offspring population size can drop below $\ln{n}$ in two different ways. Either the algorithm encounters a reset by increasing $\lambda$ beyond $\lambda_{max}$, or on the ``natural way'' by a series of successful generations decreasing the offspring population size. We begin by showing that \whp the algorithm does not encounter a reset. For a reset, a generation with offspring population size $\lambda_{max}$ must be unsuccessful, and therefore, no offspring can increase the \onemax-value. The probability of a single offspring $y$ increasing the number of one-bits is at least $\ZM(y)/(en) \leq n^{1-\eps/4}/(en)$ \cite[Lemma~2.2]{hevia2023success}. Using that $\lambda_{max} \ge n^{\eps}$, the probability of this is at most
        \begin{align*}
            &\left(1-\frac{n^{1-\eps/4}}{en}\right)^{\lambda_{max}}
            \leq \exp{\left(-\frac{\lambda_{max}}{en^{\eps/4}}\right)}
            \leq \exp{\left(-\frac{n^{\eps}}{en^{\eps/4}}\right)}
            = o\left(1/n^3\right).
        \end{align*}
        Recall that we assumed the algorithm makes $o(n\ln{n}/p)$ evaluations. Hence, the algorithm does not encounter a reset \whp.
        
        It remains to show that $\lambda$ does not ``naturally'' drop below $6eF^{16/\eps}\ln{n}$. We introduce the notion of a \emph{phase}. A phase starts as soon as $\lambda$ is reduced below $\gamma$ and ends if it is either increased back to at least $\gamma$ or falls below $6e F^{16/\eps} \ln{n}$. We assume the algorithm does not sample a distorted point during a phase as otherwise, the lemma follows immediately. We want to apply Lemma~\ref{lemma:interval-crossing} with $\alpha \coloneqq 6e F^{16/\eps} \ln{n}$ and $\beta \coloneqq 16/\eps$. To bound the probability of a successful generation, notice that each generation creates at most $\gamma$ offspring and the number of zero-bits is at most $n^{1-\eps/4}$. With a union bound, the probability that a single generation is successful is therefore at most $\gamma n^{1-\eps/4}/n = \gamma n^{-\eps/4} \eqqcolon q$. By Lemma~\ref{lemma:interval-crossing}, the probability of a phase ending due to drop in $\lambda$ below $\alpha$ is at most
        \begin{align*}
            \frac{1/q-2}{\left(1/q-1\right)^{\beta+1}-1}
            = \frac{\gamma^{-1}n^{\eps/4}-2}{\left(\gamma^{-1}n^{\eps/4}-1\right)^{16/\eps+1}-1}
            \leq \frac{n^{\eps/4}}{\left(\gamma^{-1}n^{\eps/4}-1\right)^{16/\eps+1}-1}.
        \end{align*}
        By choosing an appropriate small positive constant $\xi$ such that $\xi < (1- \eps/4)\eps/16$, this is at most
        \begin{align*}
            &\frac{n^{\eps/4}}{\left(\gamma^{-1}n^{\eps/4}-1\right)^{16/\eps+1}-1}
            = \frac{n^{\eps/4}}{\left(n^{\eps/4 - \xi} (\gamma^{-1} n^{\xi})-1\right)^{16/\eps+1}-1}\\
            &\leq \frac{n^{\eps/4}}{\left(n^{\eps/4 - \xi}\right)^{16/\eps+1}-1}
            \leq \frac{n^{\eps/4}}{\left(n^{\eps/4 - \xi}\right)^{16/\eps}}
            = \frac{n^{\eps/4}}{n^{4 - 16\xi/ \eps}}
            < n^{-3}.
        \end{align*}
        Since there are $o(n\ln{n}/p)$ evaluations and therefore also $o(n\ln{n}/p)$ phases, w.h.p. no phase will result in a drop of $\lambda$ below $6eF^{16/\eps}\ln{n}$. 
        
        It remains to show (iii). All points in $I'$ have fitness at most $n- n^{1- \eps/2} + d \leq n- k^\ast$. Therefore, the algorithm can not terminate before crossing the interval. If the algorithm enters and leaves $I'$ several times, we consider the last such time. By \cite[Lemma~3.3]{jorritsma2023comma} the probability of an offspring having a Hamming distance $\Omega(\ln{n})$ to its parent is $n^{-\Omega(\ln{}\ln{n})}$. Thus when the algorithm enters $I'$, \whp it enters in a point $x_t$ such that $\ZM(x_t) \ge n^{1- \eps/4} - O(\ln{n})$. From such a starting point, even the \oea on \onemax requires $\Omega(n \ln{n})$ evaluations to cross $I'$ \cite[Theorem~3.6]{jorritsma2023comma}. By the domination result \cite[Theorem~3.5]{jorritsma2023comma}, the \sear requires at least as many evaluations as the \oea, hence $\Omega(n \ln{n})$ evaluations to cross the interval.
    \qed
    \end{proof}
    Having established that the algorithm enters the set of distorted points $\CD$ with at least logarithmic $\lambda$ in $I'$, we proceed to show that the algorithm will not leave the set $\CD$ before crossing the following interval $I$ as well. We additionally show that the algorithm becomes elitist, \ie, the fitness is not reduced throughout.
    \begin{lemma}
        \label{lemma:stay-in-distorted-points}
        Consider the \sear on \disOM as in Theorem~\ref{thm:main}, in particular, let $\delta$ be such that $p \ge n^{\delta}/ \lambda_{max}$. Let $\eps > 0$ be a small constant such that $k^\ast + d \leq n^{1- \eps}$ and $\eps \leq \delta/4$. Let the current state  $(x_0, \lambda_0)$ satisfy $ZM(x_0) \leq n^{1- \eps/4}$, $x_0$ is distorted and $\lambda_0 \ge 6e F^{16/\eps}\ln{n}$. With high probability the algorithm neither leaves the set of distorted points $\mathcal{D}$ nor decreases the fitness in a single step until either the total number of evaluations is in $\Omega(n \ln{n}/p)$, or the distance to $\vec{1}$ is reduced to at most $n^{1- \eps}$. 
    \end{lemma}
    \begin{proof}
        If each generation has a clone among the offspring \whp the algorithm neither leaves $\mathcal{D}$ nor reduces the fitness by Lemma~\ref{lemma:accept_dist_offspring}. To show that each iteration has a clone, assume $\lambda$ does not drop below $6e\ln{n}$. By Lemma~\ref{lemma:prob-no-clone}, a single generation does not have a clone among its offspring with probability at most
        \begin{displaymath}
            \left(1- \frac{n-1}{en}\right)^{6e \ln{n}} 
            \leq n^{-6(n-1)/n}
            \leq n^{-3}.
        \end{displaymath}
        Therefore, the probability of not having a clone among the first $o(n \ln{n}/p)$ generations is $1- o(1)$ by a union bound.
        
        It remains to show that $\lambda$ does not drop below $6e \ln{n}$. We proceed almost identically as in the proof of (ii) in Lemma~\ref{lemma:big-lambda-distorted}.
        We start with the probability of a reset. If $\pimp$ is the probability of a single offspring increasing the fitness \wrt the parent, the probability of a single generation causing a reset is at most $\left(1- \pimp\right)^{\lambda_{max}}$. For an offspring to increase the fitness it must both increase the \onemax-value and be distorted. It might be tempting to bound the probability of an offspring being distorted by $p$, but we need to mind that the noise is frozen. Consequently, we do not get fresh randomness in each step. We circumvent the problem in a similar fashion to Lemma~\ref{lemma:prob-p}.

        There are $\binom{n^{1-\eps}}{3} = \Omega(n^{3-3\eps})$ points in the three-neighborhood of the parent, which additionally increase the \onemax-value by three. For sufficiently small $\eps$ any such offspring is distorted with probability $\Omega(p)$ by an analogous reasoning to Lemma~\ref{lemma:prob-p}. The probability of an offspring falling into this category is 
        \begin{displaymath}
            \binom{n^{1-\eps}}{3} \frac{1}{n^3} \left(1-\frac{1}{n}\right)^{n-3} 
            \ge \left(\frac{n^{1-\eps}}{3n}\right)^3 \left(1-\frac{1}{n}\right)^{n-1} 
            \ge \frac{n^{-3\eps}}{27e}
            = \Omega(n^{-3\eps}).
        \end{displaymath}
        Together this implies that $\pimp = \Omega(p n^{-3\eps})$. Leveraging the relationship $p \ge n^{\delta}/ \lambda_{max} \ge n^{4\eps}/ \lambda_{max}$, the probability of a reset in a single generation is at most 
        \begin{align*}
            &\left(1- \pimp\right)^{\lambda_{max}}
            \leq \exp{\left(-\Omega\left(\frac{n^{4 \eps}}{\lambda_{max}} n^{-3\eps}\right) \lambda_{max}\right)}
            \leq \exp{\left(-\Omega\left(n^{\eps}\right)\right)}.
        \end{align*}
        W.h.p. the algorithm will not encounter a reset during the $o(n \ln{n}/p)$ generations for sufficiently small $\eps$. 

        It remains to show that $\lambda$ does not drop below $6e\ln{n}$ as a consequence of a ``natural'' reduction. Using an identical reasoning to the proof of (i) in Lemma~\ref{lemma:big-lambda-distorted}, with $\alpha \coloneqq 6e \ln{n}$, $\beta \coloneqq 16/\eps$ and $q \coloneqq (6e F^{16/\eps}\ln{n}) n^{- \eps/4} \leq (6e F^{32/\eps} \ln{n}) n^{-\eps/4}$, it is clear that \whp $\lambda$ will not drop below $6e\ln{n}$.
    \qed
    \end{proof}
    As a final step, it remains to show that the algorithm takes $\Omega(n \ln{n}/p)$ evaluations to cross the second interval $I$.
    \begin{lemma}
        \label{lemma:reaching-target-in-distorted-points}
        Let $\eps > 0$ be a small constant such that $k^\ast + d \leq n^{1- \eps}$. Consider the \sear on \disOM and the interval $I \coloneqq [n^{1- \eps}, n^{1- \eps/2}]$ of Hamming distances to $\Vec{1}$. If the algorithm neither leaves the set $\CD$ of distorted points nor decreases the fitness, then with high probability it takes $\Omega(n \ln{n}/p)$ evaluations to cross the interval $I$.
    \end{lemma}
    \begin{proof}
        Due to the similar nature of the problem on hand, the following part is closely related to the analysis of the lower bound of $T^{plus}$ in \cite[Theorem~1.1]{jorritsma2023comma}. We adopt the notion of a $(1-p)$-rejection run. Such a run differs from a regular one in that each sampled search point is discarded with a probability $1-p$. We show that for any state $(x_t, \lambda_t)$ of the algorithm on \disOM, the probability of increasing the fitness by $r>0$ without leaving $\CD$ is at most the probability of increasing it by $r$ in a $(1-p)$-rejection run on \onemax.
        
        Assume an offspring $y$ satisfies $\OM(y) = \OM(x_t) + r$. If $y$ has been sampled before, it is not distorted as we assumed the algorithm does not decrease the fitness, and therefore, it would have moved to the point the first time it was sampled. Thus, $y$ is not considered for selection with probability $1$. On the other hand, if $y$ has not yet been sampled, it is distorted with probability p. Thus, it is not considered for selection with probability $1-p$. Summarizing, each point increasing the fitness by $r$ is not considered for selection with probability at least $1-p$. In other words, it is rejected with at least this probability.

        It remains to show that a $(1-p)$-rejection run of the algorithm on \onemax takes $\Omega(n \ln{n}/p)$ evaluations. By \cite[Theorem~3.6]{jorritsma2023comma} \whp the \oea takes $\Omega(n \ln{(n^{1- \eps/2}/n^{1- \eps}))} = \Omega(n \ln{n})$ evaluations to cross $I$. Hence a $(1-p)$-rejection run of the same algorithm takes $\Omega(n \ln{n}/p)$ evaluations. By the domination result \cite[Theorem~3.5]{jorritsma2023comma} this implies that the same is true for a $(1-p)$-rejection run of the \sear on \onemax.
    \qed
    \end{proof}
    We now bring everything together and prove Theorem~\ref{thm:main}.
    \begin{proof}[of Theorem~\ref{thm:main}]
        Let $\eps > 0$ be a small constant such that $k^\ast + d \leq n^{1- \eps}$ and $\eps \leq \delta/4$. Consider the intervals $I' \coloneqq [n^{1-\eps/2}, n^{1-\eps/4}]$ and $I \coloneqq [n^{1-\eps}, n^{1-\eps/2}]$ of Hamming distances to $\Vec{1}$. With high probability the initial search point of the algorithm has a distance of at least $n/3$ to $\Vec{1}$. It thus has to cross both intervals to reach the target fitness of $n-k^\ast$. By \cite[Lemma~3.3]{jorritsma2023comma}, the probability of an offspring having Hamming distance at least $c \ln{n}$ to its parent is $n^{-\Omega{(\ln{\ln{n}})}}$. Therefore, the algorithm will not increase the \onemax-value by $\Omega(\ln{n})$ in the first $o(n \ln{n}/p)$ evaluations. As a result, the algorithm will jump over neither of the two intervals (both have size $\omega(\ln{n})$).

        Let a \emph{trial} be defined as a sequence of generations, which starts as soon as the algorithm samples a search point in $I'$ and ends if either
        \begin{enumerate}[(i)]
            \item the algorithm accepts an offspring $y$ with $\ZM(y) > n^{1-\eps/4}$, or 
            \item the algorithm accepts an offspring $y$ with $\ZM(y) < n^{1-\eps/2}$, or 
            \item the algorithm reaches a  state $(x_t, \lambda_t)$ such that $\ZM(x_t) \in I'$, $x_t$ is distorted and $\lambda_t \ge 6e F^{16/\eps}$.
        \end{enumerate}
        If a trial ends due to condition (i), with an analogous argumentation to before, a new trial will start again \whp. Under the assumption of a trial not ending due to condition (i), by Lemma~\ref{lemma:big-lambda-distorted}, \whp the trial will end due to condition (iii), or the number of evaluations is in $\Omega(n \ln{n}/p)$, in which case the theorem would follow immediately.

        From such a state \whp the algorithm will neither decrease the fitness nor leave the set of distorted points $\CD$ before either the total number of evaluations is in $\Omega{(n \ln{n}/p)}$ or the algorithm accepts an offspring $y$ with $\ZM(y) < n^{1- \eps}$, by Lemma~\ref{lemma:stay-in-distorted-points}. This allows us to apply Lemma~\ref{lemma:reaching-target-in-distorted-points}. To cross the interval $I$ the algorithm will require $\Omega(n \ln{n}/p)$ evaluations, which concludes the proof.
    \qed
    \end{proof}
    Corollary~\ref{cor:results} puts our findings in context with \cite{jorritsma2023comma}. Items (\ref{cor:results_item_2}) and (\ref{cor:results_item_3}) are direct results of Theorem~\ref{thm:main} and \cite[Theorem~1.1]{jorritsma2023comma} respectively. Even though (\ref{cor:results_item_1}) is similar to the corresponding result in \cite[Theorem~1.1]{jorritsma2023comma}, it is not a direct consequence. We provide a brief proof sketch.
    \begin{proof}[of Corollary~\ref{cor:results}~(\ref{cor:results_item_1})]
        We define the same intervals $I$ and $I'$ as in the proof of Theorem~\ref{thm:main}. Since even the \oea takes $\Omega(n \ln{n})$ evaluations to cross $I'$, so does the \oclea. With Lemma~\ref{lemma:prob-p}, among these offspring at least one is distorted, which is also accepted \whp by Lemma~\ref{lemma:accept_dist_offspring}. Since no offspring among the first $o(n \ln{n}/p)$ increases the Hamming distance by $\Omega(\ln{n})$, the algorithm will not leave the set of distorted points $\CD$. To cross the second interval $I$, the algorithm requires $\Omega(n \ln{n}/p)$ evaluations with analogous reasoning to Lemma~\ref{lemma:reaching-target-in-distorted-points}.
    \end{proof}
    
\section{Experiments}\label{sec:experiments}
    We corroborate our theoretical results empirically with two sets of experiments\footnote{The code for the experiments can be found at \href{https://github.com/kosturm/EAs-on-Distorted-OneMax}{https://github.com/kosturm/EAs-on-Distorted-OneMax}}. We theoretically showed a lower bound on the runtime $T$ of the \sear of $\Omega(n \ln{n}/p)$, so we plot the normalized runtime $T/((n\ln n)/p)$ in Figure~\ref{fig:plot_diff_p}. We show the mean and the standard deviation for three problem sizes with varying distortion probabilities $p$ over $50$ runs each. The remaining parameters stay unchanged. Indeed, for larger $p$, the curve is almost horizontal, which suggests that the lower bound is tight and $T=\Theta(n\ln n/p)$. For smaller $p$, the runtime becomes irregular, showing that a lower bound on $p$ as in Theorem~\ref{thm:main} is indeed necessary.
\begin{figure}[ht]
    \centering
    \begin{minipage}[t]{0.49\textwidth}
        \centering
        \includegraphics[width=\linewidth]{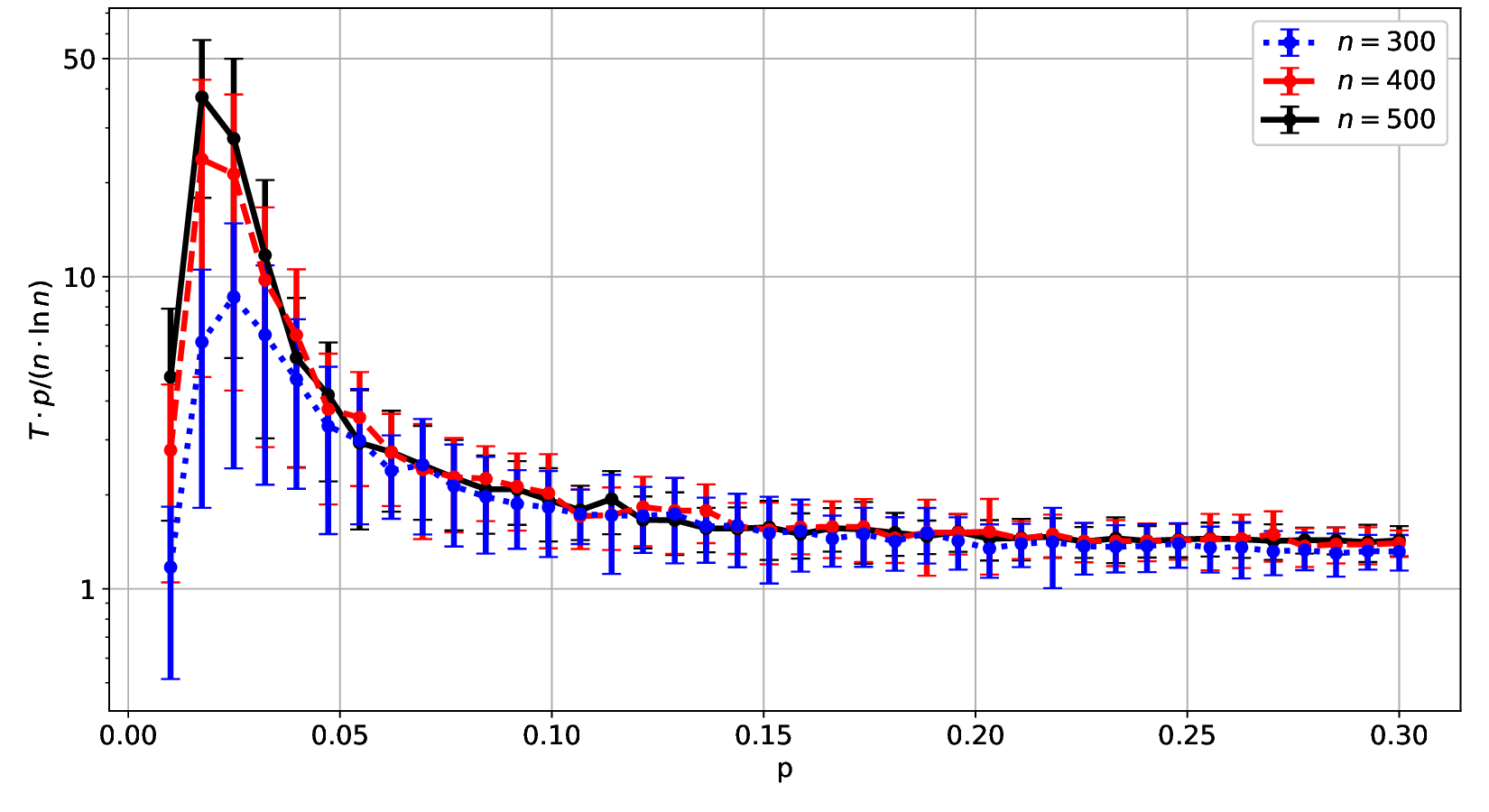}
        \caption{Normalized number of evaluations required by the \sear to optimize \disOM for different distortion probabilities $p$. We set $d= \ln{n}$, $k^\ast = n^{0.4}$, $F=1.5$, $s = 1$, $\lambda_{max} = n \ln{n}$ and average over $50$ runs each. The cutoff of $10^7$ evaluations was never reached. Note that the y-axis shows the number of evaluations $T$ multiplied by $p /(n \ln{n})$ and is scaled logarithmically.}
        \label{fig:plot_diff_p}
    \end{minipage}\hfill
    \begin{minipage}[t]{0.49\textwidth}
        \centering
        \includegraphics[width=\linewidth]{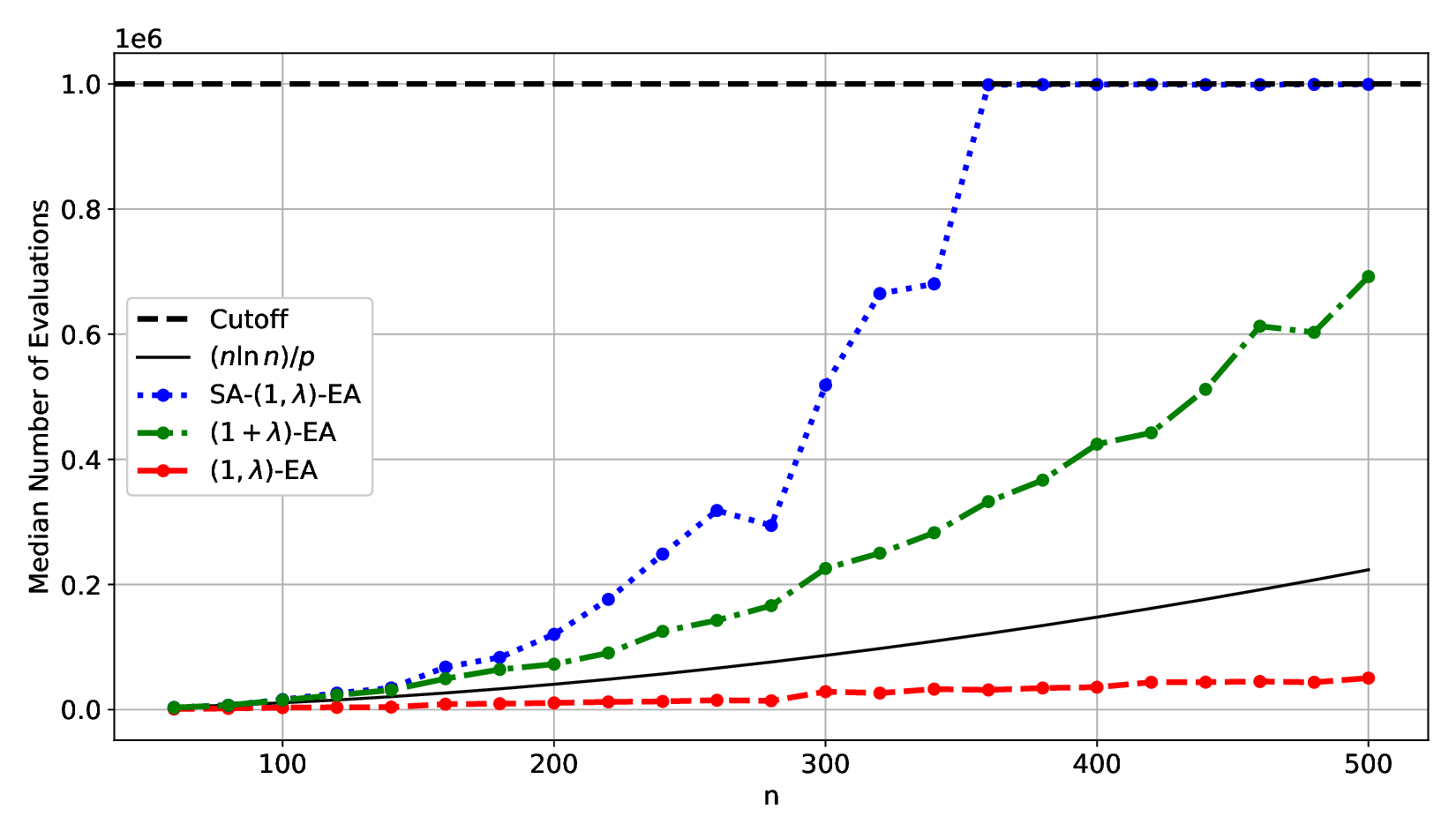}
        \caption{We take the median over 50 runs for the \oclea, the \olea and the \sear. We set $d = \ln{n}$, $k^\ast = n^{0.4}$, $\lambda_{com, plus} = \lfloor 1.5 \ln{n}\rceil$ for the \oclea and the \olea, $p = (e/(e-1))^{-\lambda_{com, plus}}$, $F= 1.5$, $s= 1$, $\lambda_{max} = n \ln{n}$. We make a cutoff at $10^{6}$ evaluations. We additionally plot the curve $n \ln{n}/p$ for reference.}
        \label{fig:plot_algo}
    \end{minipage}
\end{figure}

In Figure~\ref{fig:plot_algo}, we compare the runtime behavior of the \sear, the \oclea, and the \olea. This confirms that the \olea with static $\lambda$ is much faster than the two other algorithms.

\section{Conclusion}
    We have investigated the \sea with resets on \disOM. While this algorithm has been very successful on hill-cimbing tasks and on the multimodal function \cliff, we have shown that this does not extend to the type of local optima that \disOM represents. We believe that it is worthwhile to explore the algorithm on other theoretical benchmarks to understand better in which situations it is slowed down. Candidates include \hurdle, the recently introduced benchmark \BBFunnel~\cite{dang2021non}, and the multimodal landscapes introduced by Jansen and Zarges~\cite{jansen2016example}. Moreover, it is important to explore other self-adaptation mechanisms that may provide alternatives to the resetting mechanism studied in this paper and which may be able to keep the advantages of comma selection for local optima of the type as in \disOM. 

\subsubsection{Acknowledgements}
    This research was strongly influenced by the discussions at the Dagstuhl seminars 22081 ``Theory of Randomized Optimization Heuristics'' and 23332 ``Synergizing Theory and Practice of Automated Algorithm Design for Optimization''.

\bibliographystyle{splncs04}
\bibliography{refs}

\end{document}